\documentclass[runningheads]{llncs}

\usepackage{amssymb}
\usepackage{amsmath} 
\usepackage{import}
\usepackage{thm-restate, upgreek, dsfont, bm}
\usepackage{cite}
\usepackage[obeyspaces,spaces]{url}
\usepackage{graphicx, subcaption, tikz}
    \usetikzlibrary{arrows,topaths}
\newcommand{\D}{\mathcal{D}}
\newcommand{\B}{\mathcal{B}}

\newcommand{\X}{\mathcal{X}}
\newcommand{\Y}{\mathcal{Y}}
\renewcommand{\L}{\mathcal{L}}
\newcommand{\FRo}{F_{R_o}}
\newcommand{\fro}{f_{R_o}}
\newcommand{\frol}{f_{R_o + l}}
\newcommand{\froL}{f_{R_{o}+L}}
\newcommand{\FRoL}{F_{R_{o}+L}}
\newcommand{\s}{source}
\renewcommand{\r}{recipient}

\newcommand{\lk}{learned knowledge}

\newcommand{\lrs}{learning resources}

\newcommand{\dd}{decomposable}

\newcommand{\dm}{decomposable probability-of-success metric}
\newcommand{\dms}{decomposable probability-of-success metrics}

\renewcommand{\D}{\mathcal{D}}

\newcommand{\A}{\mathcal{A}}

\renewcommand{\L}{\mathcal{L}}

\newcommand{\M}{Monta{\~{n}}ez}
\newcommand{\g}{\phi}
\newcommand{\E}{\mathbb{E}}
\newcommand{\unif}{\mathcal{U}[\mathcal{P}]}

\newcommand{\dif}{\text{d}}
\newcommand{\af}{\mathrm{Affin}}

\usepackage{array}
\newcolumntype{P}[1]{>{\centering\arraybackslash}p{#1}}

\begin{document}

\title{Limits of Transfer Learning}

 \author{Jake Williams\inst{1}
 \and
 Abel Tadesse\inst{2} 
 \and
 Tyler Sam\inst{1} 
 \and
 Huey Sun\inst{3}
 \and \\
George D.\ Monta{\~n}ez\inst{1}
}
\authorrunning{J. Williams et al.}
\institute{Harvey Mudd College, Claremont, CA 91711, USA \and
 Claremont McKenna College, Claremont, CA 91711, USA
 \and
 Pomona College, Claremont, CA 91711, USA\\
 \email{jwilliams@hmc.edu}}

\maketitle              
\begin{abstract}
Transfer learning involves taking information and insight from one problem domain and applying it to a new problem  domain. Although widely used in practice, theory for transfer learning remains less well-developed. To address this, we prove several novel results related to transfer learning, showing the need to carefully select which sets of information to transfer and the need for dependence between transferred information and target problems. Furthermore, we prove how the degree of probabilistic change in an algorithm using transfer learning places an upper bound on the amount of improvement possible. These results build on the algorithmic search framework for machine learning, allowing the results to apply to a wide range of learning problems using transfer.

\keywords{Transfer Learning \and Algorithmic Search Framework \and Affinity.}
\end{abstract}

\section{Introduction}
Transfer learning is a type of machine learning where insight gained from solving one problem is applied to solve a separate, but related problem~\cite{pan2009survey}. Currently an exciting new frontier in machine learning, transfer learning has diverse practical application in a number of fields, from training self-driving cars~\cite{choi2018driving}, where model parameters are learned in simulated environments and transferred to real-life contexts, to audio transcription~\cite{wang2015transfer}, where patterns learned from common accents are applied to learn less common accents. Despite its potential for use in industry, little is known about the theoretical guarantees and limitations of transfer learning.

To analyze transfer learning, we need a way to talk about the breadth of possible problems we can transfer from and to under a unified formalism. One such approach is the reduction of various machine learning problems (such as regression and classification) to a type of search, using the method of the algorithmic search framework \cite{Montanez2016TheFO,montanez2017machine}. This reduction allows for the simultaneous analysis of a host of different problems, as results proven within the framework can be applied to any of the problems cast into it. In this work, we show how transfer learning can fit within the framework, and define affinity as a measure of the extent to which information learned from solving one problem is applicable to another. Under this definition, we prove a number of useful theorems that connect affinity with the probability of success of transfer learning. We conclude our work with applied examples and suggest an experimental heuristic to determine conditions under which transfer learning is likely to succeed.

\section{Distinctions from Prior Work}
Previous work within the algorithmic search framework has focused on bias~\cite{montanez2019fobfl,lauw2020bias}, a measure of the extent to which a distribution of information resources is predisposed towards a fixed target. The case of transfer learning carries additional complexity as the recipient problem can use not only its native information resource, but the learned information passed from the source as well. Thus, affinity serves as an analogue to bias which expresses this nuance, and enables us to prove a variety of interesting bounds for transfer learning.

\section{Background}

\subsection{Transfer Learning}
\subsubsection{Definition of Transfer Learning}

Transfer learning can be defined by two machine learning problems~\cite{pan2009survey}, a source problem and a recipient problem. Each of these is defined by two parts, a domain and a task. The domain is defined by the feature space, $\X$, the label space, $\Y$, and the data, $D = \{(x_i, y_i), \dots, (x_n, y)\}$, where $x_i \in \X$ and $y_i \in \Y$. The task is defined by an objective function $P_f(Y|X)$, which is a conditional distribution over the label space, conditioned on an element of the feature space. In other words, it tells us the probability that a given label is correct for a particular input. A machine learning problem is ``solved'' by an algorithm $\mathcal{A}$, which takes in the domain and outputs a function $P_{\mathcal{A}}(Y|X)$. The success of an algorithm is its ability to learn the objective function as its output. Learning and optimization algorithms use a loss function $\L(p)$ to evaluate an output function to decide if it is worthy of outputting. Such algorithms can be viewed as black-box search algorithms~\cite{Montanez2016TheFO}, where the particular algorithm determines the behavior of the black box. For transfer learning under this view, the output is defined as the final element in the search history.

\subsubsection{Types of Transfer Learning}

Pan and Yang separated transfer learning into four categories based on the type of information passed between domains~\cite{pan2009survey}:
\begin{itemize}
    \item {\it Instance transfer}: Supplementing the target domain data with a subset of data from the source domain.
    \item {\it Feature-representation transfer}: Using a feature-representation of inputs that is learned in the source domain to minimize differences between the source and target domains and reduce generalization error in the target task.
    \item {\it Parameter transfer}: Passing a subset of the parameters of a model from the source domain to the target domain to improve the starting point in the target domain.
    \item {\it Relational-knowledge transfer}: Learning a relation between knowledge in the source domain to pass to the target domain, especially when either or both do not follow i.i.d.\ assumptions. 
\end{itemize}

\subsection{The Search Framework}
\begin{figure}[htbp!]
    \centering
    \includegraphics[scale=0.4]{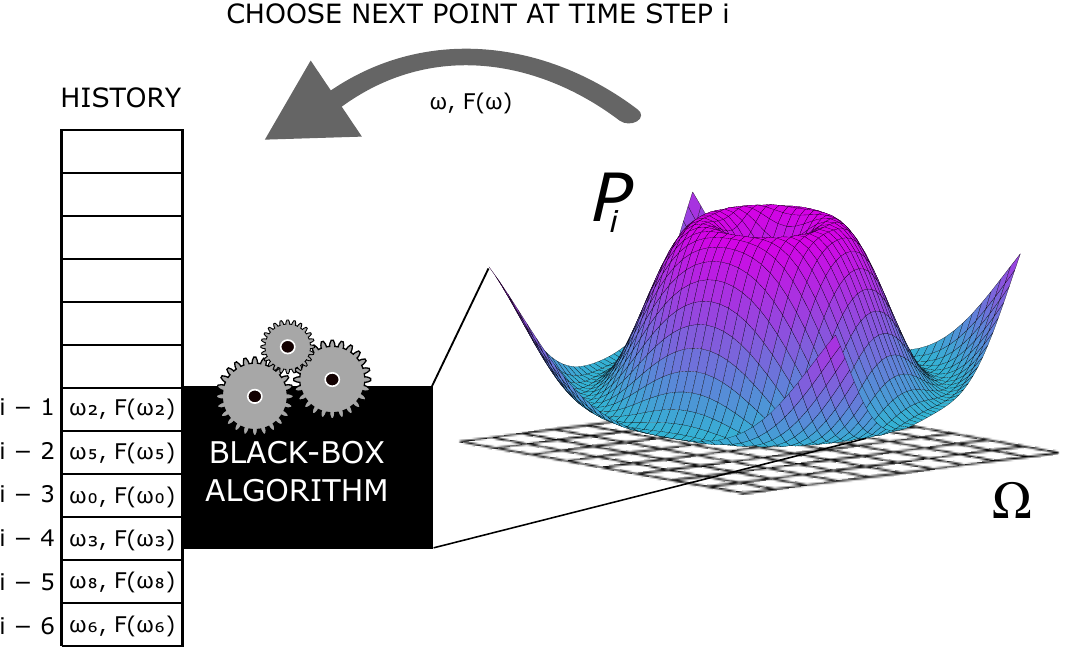}
    \caption{Black-box search algorithm. We add evaluated queries to the history according to the distribution iteratively. Reproduced from~\cite{Montanez2016TheFO}.}
    \label{fig:searchFramework}
\end{figure}
To analyze transfer learning from a theoretical perspective, we take inspiration from previous work that views machine learning as a type of search. \M{} casts machine learning problems, including Vapnik's general learning problem (covering regression, classification, and density estimation) into an algorithmic search framework  \cite{montanez2017machine}. For example, classification is seen as a search through all possible labelings of the data, and clustering as a search through all possible ways to cluster the data \cite{montanez2017machine}. This framework provides a common structure which we can use to analyze different machine learning problems, as each of them can be seen as a search problem with a defined set of components. Furthermore, any result we prove about search problems applies to all machine learning problems we can represent within the framework. 

Within the algorithmic search framework, the three components of a search problem are the search space $\mathrm{\Omega}$, target set $T$, and external information resource $F$. The search space, which is finite and discrete due to the finite precision representation of numbers on computers, is the set of elements to be examined. The target set is a nonempty subset of $\mathrm{\Omega}$ that contains the elements we wish to find. Finally, the external information resource is used to evaluate the elements of the search space. Usually, the target set and external information resource are related, as the external information resource guides the search to the target~\cite{Montanez2016TheFO}.    

In this framework, an iterative algorithm searches for an element in the target set, depicted in Figure~\ref{fig:searchFramework}. The algorithm is viewed as a black box that produces a probability distribution over the search space from the search history. At each step, an element is sampled from $\mathrm{\Omega}$ according to the most recent probability distribution. The external information resource is then used to evaluate the queried element, and the element and its evaluation are added to the search history. Thus, the search history is the collection of all points sampled and all information gained from the information resource during the course of the search. Finally, the algorithm creates a new probability distribution according to its rules. Abstracting the creation of the probability distribution allows the search framework to work with many different search algorithms~\cite{montanez2017machine}.  

\subsection{Decomposable Probability-of-Success Metrics}
Working within the same algorithmic search framework~\cite{Montanez2016TheFO,montanez2017machine, montanez2019fobfl},  to measure the performance of search and learning algorithms, Sam et al.~\cite{sam2020decomposable} defined \dms{} as
\[
    \g(t,f) = \mathbf{t}^{\top}\mathbf{P}_{\g,f}=P_{\g}(X \in t|f)
    \]
where $\mathbf{P}_{\g,f}$ is not a function of target set $t$ (with corresponding target function $\mathbf{t}$), being conditionally independent of it given information resource $f$. They note that one can view $\mathbf{t}^{\top}\mathbf{P}_{\g,f}$ as an expectation over the probability of successfully querying an element from the target set at each step according to an arbitrary distribution. In the case of transfer learning, the distribution we choose should place most or all of its weight on the last or last couple of steps -- since we transfer knowledge from the source problem's model after training, we care about our success at the last few steps when we're done training, rather than the first few.

\subsection{Casting Transfer Learning into the Search Framework}
 Let $\mathcal{A}$ denote a fixed learning algorithm.
 We cast the \s{}, which consists of $\mathcal{X}_s, \mathcal{Y}_s, D_s,$ and $P_{f, s}(Y|X)$, into the algorithmic search framework as 
\begin{enumerate}
     \item $\mathrm{\Omega} = $ range$(\A)$;
    \item $T = \{P(Y|X) \in \textrm{range($\A$)} \mid \Xi(P, P_{f, s}) < \epsilon \}$;
    \item $F = \{D_s, \mathcal{L}_s \}$;
    \item $F(\emptyset) = \emptyset; \textrm{ and }$
    \item $F(\omega_i) = \mathcal{L}_s(\omega_i)$.
\end{enumerate}
where $w_i$ is the $i$th query in the search process, $\mathcal{L}_s$ the loss function for the \s, and $\Xi_s$ is an error functional on learned conditional distribution $P$ and the optimal conditional distribution $P_{f, s}$. 

Generally, any information from the \s{} can be encoded in a binary string, so we represent the knowledge transferred as a finite length binary string. Let this string be $L = \{0, 1\}^n$. Thus, we cast the \r{}, which consists of $\mathcal{X}_r, \mathcal{Y}_r, D_r,$ and $P_{f, r}(Y|X)$, into the search framework as 
\begin{enumerate}
     \item $\mathrm{\Omega} = $ range$(\A)$;
    \item $T = \{P(Y|X) \in \textrm{range($\A$)} \mid \Xi(P, P_{f, r}) < \epsilon \}$;
    \item $F = \{D_r, \mathcal{L}_r \}$;
    \item $F(\emptyset) = L; \textrm{ and }$
    \item $F(\omega_i) = \mathcal{L}_r(\omega_i)$.
\end{enumerate}
where $\mathcal{L}_r$ is a loss function, and $\Xi_t$ is an error functional on $P$ and the optimal conditional distribution $P_{f, r}$.

\begin{figure*}[ht]
\centering
\begin{subfigure}[b]{0.3\textwidth}
\begin{tikzpicture}[>=stealth',semithick,auto]
    \tikzstyle{subj} = [circle, minimum width=8pt, fill, inner sep=0pt]
    \tikzstyle{obj}  = [circle,minimum width=8pt, fill, inner sep=0pt]
    \tikzstyle{dc}   = [circle, minimum width=8pt, draw, inner sep=0pt, path picture={\draw (path picture bounding box.south east) -- (path picture bounding box.north west) (path picture bounding box.south west) -- (path picture bounding box.north east);}]

    \tikzstyle{every label}=[font=\bfseries]

    \node[obj,  label=right:$T_R$] (tr) at (0,2) {};
    \node[obj, label=right:$F_{R_o}$] (fro) at (0,1) {};
    \node[obj,   label=right:$F_{R_o+L}$] (frok) at (0,0) {};
    \node[obj,  label=below:$L$] (k) at (-1,0) {};
    \node[obj,  label=below:$F_S$] (fs) at (-2,0) {};

    \path[->]   (tr)    edge               (fro);
    \draw[dashed, ->, inner sep=0pt]   (tr)    edge  node {? }           (fs);
    \path[->]   (fro)   edge                (frok);
    \path[->]   (fs)   edge                (k);
    \path[->]   (k)   edge                (frok);

\end{tikzpicture}
    \caption{General Case}
    \label{fig:gc}
    \end{subfigure}%
    ~
    \begin{subfigure}[b]{0.3\textwidth}
        \begin{tikzpicture}[>=stealth',semithick,auto]
    \tikzstyle{subj} = [circle, minimum width=8pt, fill, inner sep=0pt]
    \tikzstyle{obj}  = [circle,minimum width=8pt, fill, inner sep=0pt]
    \tikzstyle{dc}   = [circle, minimum width=8pt, draw, inner sep=0pt, path picture={\draw (path picture bounding box.south east) -- (path picture bounding box.north west) (path picture bounding box.south west) -- (path picture bounding box.north east);}]

    \tikzstyle{every label}=[font=\bfseries]

    \node[obj,  label=right:$T_R$] (tr) at (0,2) {};
    \node[obj, label=right:$F_{R_o}$] (fro) at (0,1) {};
    \node[obj,   label=right:$F_{R_o+L}$] (frok) at (0,0) {};
    \node[obj,  label=below:$L$] (k) at (-1,0) {};
    \node[obj,  label=below:$F_S$] (fs) at (-2,0) {};

    \path[->]   (tr)    edge               (fro);
    \path[->]   (tr)    edge               (fs);
    \path[->]   (fro)   edge                (frok);
    \path[->]   (fs)   edge                (k);
    \path[->]   (k)   edge                (frok);

\end{tikzpicture}
    \caption{Case 1}
    \end{subfigure}%
    ~
    \begin{subfigure}[b]{0.3\textwidth}
        \centering
        \begin{tikzpicture}[>=stealth',semithick,auto]
    \tikzstyle{subj} = [circle, minimum width=8pt, fill, inner sep=0pt]
    \tikzstyle{obj}  = [circle,minimum width=8pt, fill, inner sep=0pt]
    \tikzstyle{dc}   = [circle, minimum width=8pt, draw, inner sep=0pt, path picture={\draw (path picture bounding box.south east) -- (path picture bounding box.north west) (path picture bounding box.south west) -- (path picture bounding box.north east);}]

    \tikzstyle{every label}=[font=\bfseries]

    \node[obj,  label=right:$T_R$] (tr) at (0,2) {};
    \node[obj, label=right:$F_{R_o}$] (fro) at (0,1) {};
    \node[obj,   label=right:$F_{R_o+L}$] (frok) at (0,0) {};
    \node[obj,  label=below:$L$] (k) at (-1,0) {};
    \node[obj,  label=below:$F_S$] (fs) at (-2,0) {};

    \path[->]   (tr)    edge               (fro);
    \path[->]   (fro)   edge                (frok);
    \path[->]   (fs)   edge                (k);
    \path[->]   (k)   edge                (frok);

\end{tikzpicture}
    \caption{Case 2}
    \end{subfigure}
    \caption{Dependence Structure for Transfer Learning}
    \label{DSTL}
\end{figure*}
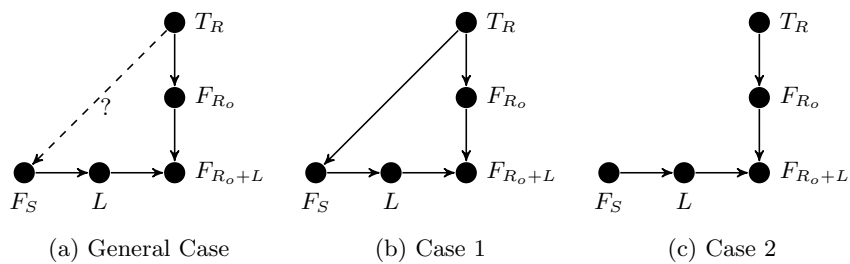

\section{Preliminaries}

\subsection{Affinity}

In a transfer learning problem, we want to know how the \s{} problem can improve the \r{} problem, which it does through the information resource. So, we can think about how the bias of the \r{} is changed by the \lk{} that the \s{} passes over. Recall that the bias is defined by a distribution over possible information resources. However, we know that the information resource will contain $\fro$, the original information resource from the recipient problem. Our distribution over information resources will therefore take that into account, and only care about the \lk{} being passed over from the \s. 

To quantify this, we let $\D_L$ be the distribution placed over $L$, the possible \lrs, by the \s. We can use it to make statements similar to bias in traditional machine learning by defining a property called \textit{affinity}.

Consider a transfer learning problem with a fixed $k$-hot target vector $\mathbf{t}$, fixed \r{} information resource $\fro$, and a distribution $\D_L$ over a collection of possible \lrs{}, with $L \sim \D_L$. The \textbf{affinity} between the distribution and the \r{} problem is defined as
\begin{align*}
    \text{Affin}(\D_L, \textbf{t}, \fro) &= \E_{\D_L}[\textbf{t}^{\top}{\bf P}_{\g,\froL}] - \textbf{t}^{\top}{\bf P}_{\g,\fro} \\
&= \textbf{t}^{\top}\E_{\D_L}[{\bf P}_{\g,\froL}] - \g(t , \fro) \\
&= \textbf{t}^{\top}\int_{\mathcal{B}} {\bf P}_{\g,\froL} \D_L(l) \text{d}l - \g(t,\fro).
\end{align*}
Affinity can be interpreted as the expected increase or decrease in performance on the recipient problem when using a learning resource sampled from a set according to a given distribution.

Using affinity, we seek to prove bounds similar to existing bounds about bias, such as the Famine of Favorable Targets and Famine of Favorable Information Resources~\cite{montanez2019fobfl}.

\section{Theoretical Results}
We begin by showing that affinity is a conserved quantity, implying that positive affinity towards one target is offset by negative affinity towards other targets.
\begin{restatable}[Conservation of Affinity]{theorem}{consofaffin}\label{thm:CONS-OF-AFFIN}
For any arbitrary distribution $\D$ and any $\fro$,
\[
    \sum_{\mathbf{t}} \af(\D, \mathbf{t}, \fro) = 0.
\]
\end{restatable}
This result agrees with other no free lunch~\cite{Wolpert1997NoFL} and conservation of information results~\cite{SchafferConservation,dembski2009conservation,lauw2020bias}, showing that trade-offs must always be made in learning.

Assuming the dependence structure of Figure~\ref{DSTL}, we next bound the mutual information between our updated information resource and the recipient target in terms of the source and recipient information resources.

\begin{restatable}[Transfer Learning under Dependence]{theorem}{tld}\label{TLuD}
Define 
\[
\phi_{TL} := \E_{T_R,F_{R+L}}[\phi(T_R,F_{R+L})] = \Pr(\omega \in T_R; \mathcal{A}) 
\] as the probability of success for transfer learning. Then, 
    \begin{align*}
    \phi_{TL} \leq \frac{I(F_S; T_R) + I(F_R; T_R) + D(P_{T_R} \| \mathcal{U}_{T_R}) + 1}{I_{\mathrm{\Omega}}}
\end{align*}
 where $I_{\mathrm{\Omega}} = -\log |T_R|/|\mathrm{\Omega}|$ ($T_R$ being of fixed size), $D(P_{T_R} \| \mathcal{U}_{T_R})$ is the Kullback-Leibler divergence between the marginal distribution on $T_R$ and the uniform distribution on $T_R$, and $I(F; T)$ is the mutual information. 
\end{restatable}
This theorem upper bounds the probability of successful transfer ($\phi_{TL}$) to show that transfer learning can't help us more than our information resources allow. This point is determined by $I(F_S; T_R)$, the amount of mutual information between the source's information resource and the recipient's target, by $I(F_R; T_R)$ the amount of mutual information between the \r's information resource and the recipient's target, and how much $P_{T_R}$ (the distribution over the \r's target) `diverges' from the uniform distribution over the \r's target, $\mathcal{U}_{T_R}$. This makes sense in that
 \begin{itemize}
     \item the more dependent $F_S$ and $T_R$, the more useful we expect the source's information resource to be in searching for $T_R$, in which case $\phi_{TL}$ can take on larger values.
     \item the more $P_{T_R}$ diverges from $\mathcal{U}_{T_R}$, the less helpless we are against the randomness (since the uniform distribution maximizes entropic uncertainty).
 \end{itemize}

\begin{restatable}[Famine of Favorable Learned Information Resources]{theorem}{fofli}
        Let $\mathcal{B}$ be a finite set of learning resources and let $t \subseteq \Omega$ be an arbitrary fixed $k$-size target set. Given a \r{} problem $(\Omega, t, \fro)$, define 
        \begin{align*}
          \mathcal{B}_{\phi_{\mathrm{min}}} &= \{l \in \mathcal{B} \mid \g(t,\frol) \geq \phi_{\mathrm{min}} \},
        \end{align*}
        where $\g(t, \frol)$ is the \dm{} for algorithm $\A$ on search problem $(\Omega, t,\frol)$ and $\phi_{\mathrm{min}} \in (0,1]$ represents the minimally acceptable probability of success under $\phi$. Then,
        \begin{align*}
            \frac{|\B_{\phi_{\mathrm{min}}}|}{|\B|} &\leq \frac{\g(t, \fro) +  \af(\mathcal{U}[\B], \mathbf{t}, \fro)}{\phi_{\mathrm{min}}}
        \end{align*}
        where $\phi(t, \fro)$ is the \dm{} with the \r{}'s original information resource.
        \label{thm:fofli}
    \end{restatable}

Theorem~\ref{thm:fofli} demonstrates the proportion of $\phi_{\mathrm{min}}$-favorable information resources for transfer learning is bounded by the degree of success without transfer, along with the affinity (average performance improvement) of the set of resources as a whole. Highly favorable transferable resources are rare for difficult tasks, within any neutral set of resources lacking high affinity. Unless a set of information resources is curated towards a specific transfer task by having high affinity towards it, the set will not and cannot contain a large proportion of highly favorable elements.

\begin{restatable}[Futility of Affinity-Free Search]{theorem}{futility}\label{thm:futilityafs}
    For any fixed algorithm $\mathcal{A}$, fixed \r{} problem $(\Omega, t, \fro)$, where $t \subseteq \Omega$ with a corresponding target function $\mathbf{t}$, and distribution over information resources $\D_L$, if $\af(\D_L, \mathbf{t}, \fro) = 0$, then
    \begin{align*}
        \Pr(\omega \in t; \A_L)
        &= \phi(t, \fro)
    \end{align*}
    where $\Pr(\omega \in t; \A_L)$ represents the expected \dd{} probability of successfully sampling an element of $t$ using $\A$ with transfer, marginalized over learning resources $L \sim \D_L$, and $\phi(t, \fro)$ is the probability of success without $L$ under the given \dd{} metric.
    \end{restatable}

Theorem~\ref{thm:futilityafs} tells us that transfer learning only helps in the case that we have a favorable distribution on learning resources, tuned to the specific problem at hand. Given a distribution \textit{not} tuned in favor of our specific problem, we can perform no better than if we had not used transfer learning. \textbf{This proves that transfer learning is not inherently beneficial in and of itself}, unless it is accompanied by a favorably tuned distribution over resources to be transferred. A natural question is how rare such favorably tuned distributions are, which we next consider in Theorem~\ref{thm:fofad}.
    
\begin{restatable}[Famine of Favorable Affinity Distributions]{theorem}{fofad}\label{thm:fofad}
Given a fixed target function $\mathbf{t}$ and a finite set of learned information resources $\mathcal{B}$, let $$\mathcal{P} = \{ \mathcal{D} \mid \mathcal{D} \in \mathbb{R}^{|\mathcal{B}|}, \sum_{l \in \B} \mathcal{D}(l) = 1 \}$$ be the set of all discrete $|\mathcal{B}|$-dimensional simplex vectors. Then,
\begin{align*}
    \frac{\mu (\mathcal{G}_{\mathbf{t}, \phi_{\mathrm{min}}})}{\mu (\mathcal{P})} &\leq \frac{\phi(t,\fro) + \af(\mathcal{U}[\B], \mathbf{t}, \fro)}{\phi_{\mathrm{min}}}
\end{align*}where $\mathcal{G}_{\mathbf{t}, \phi_{\mathrm{min}}} = \{ \mathcal{D} \mid \mathcal{D} \in \mathcal{P},\af(\mathcal{D}, \mathbf{t}, \fro) \geq \phi_{\mathrm{min}} \}$ and $\mu$ is Lebesgue measure. 
\end{restatable}
We find that highly favorable distributions are quite rare for problems that are difficult without transfer learning, unless we restrict ourselves to distributions over sets of highly favorable learning resources. (Clearly, finding a favorable distribution over a set of good options is not a difficult problem.) Additionally, note that we have recovered the same bound as in Theorem~\ref{thm:fofli}.

\begin{restatable}[Success Difference from Distribution Divergence]{theorem}{sa}
Given the performance of a search algorithm on the recipient problem in the transfer learning case, $\g_{TL}$, and without the learning resource, $\g_{NoTL}$, we can upper bound the absolute difference as
\begin{align*}
    |\g_{TL} - \g_{NoTL}| &\leq |T|\sqrt{\frac{1}{2}D_{KL}(\mathbf{P}_{TL}||\mathbf{P}_{NoTL})}.
\end{align*} 
\end{restatable}
This result shows that unless using the learning resource significantly changes the resulting distribution over the search space, the change in performance from transfer learning will be minimal.

\section{Examples and Applications}

\subsection{Examples}

We can use examples to evaluate our theoretical results. To demonstrate how Theorem \ref{TLuD} can apply to an actual case of machine learning, we can construct a pair of machine learning problems in such a way that we can properly quantify each of the terms in the inequality, allowing us to show how the probability of successful search is directly affected by the transfer of knowledge from the source problem.

Let $\mathrm{\Omega}$ be a $16 \times 16$ grid and $|T| = k = 1$. In this case, we know that the target set is a single cell in the grid, so choosing a target set is equivalent to choosing a cell in the grid. Let the distribution on target sets $P_T$ be uniformly random across the grid. For simplicity, we will assume that there is no information about the target set in the information resource, and that any information will have to come via transfer from the source problem. Thus, $I(F_R; T_R) = 0$. 

First, suppose that we provide no information through transfer,  meaning that a learning algorithm can do no better than randomly guessing. The probability of successful search will be $1/256$. We can calculate the bound from our theorem using the known quantities:
\begin{itemize}
\item $I(F_S; T_R) = 0$;
\item $I(F_R; T_R) = 0$;
\item $H(T) = 8$ (because it takes 4 bits to specify a row and 4 bits to specify a column) 
\item $D(P_{T_R} \| \mathcal{U}_{T_R}) = \log_2 \binom{256}{1} - H(T) = 8 - 8 = 0$;
\item $I_{\mathrm{\Omega}} = -\log 1/256 = 8$;
\end{itemize}
Thus, we upper bound the probability of successful search at $1/8$. 

Now, suppose that we had an algorithm which had been trained to learn which half, the top or bottom, our target set was in. This is a relatively easier task, and would be ideal for transfer learning. Under these circumstances, the actual probability of successful search doubles to $1/128$. We can examine the effect that this transfer of knowledge has on our probability of success.
\begin{itemize}
    \item $I(F_S; T_R) = H(T_R) - H(T_R | F_S) = 8 - 7 = 1$;
    \item $I(F_R; T_R) = 0$;
    \item $H(T) = 8$;
    \item $D(P_{T_R} \| \mathcal{U}_{T_R}) = 0$;
    \item $I_{\mathrm{\Omega}} = 8$;
\end{itemize}
The only change is in the mutual information between the the recipient target set and the source information resource, which was able to perfectly identify which half the target set was in. This brings the probability of successful search to $1/4$, exactly twice as high as without transfer learning. 

This result is encouraging, because it demonstrates that the upper bound for transfer learning under dependence is able to reflect changes in the use of transfer learning and their effects. The upper bound being twice as high when the probability of success is doubled is good. However, the bound is very loose. In both cases, the bound is 32 times as large as the actual probability of success. Tightening the bound may be possible; however, as seen in this example, the bound we have can already serve a practical purpose.

\subsection{Transferability Heuristic}

Our theoretical results suggest that we cannot expect transfer learning to be successful without careful selection of transferred information. Thus, it is imperative to identify instances in which transferred resources will raise the probability of success. In this section, we explore a simple heuristic indicating conditions in which transfer learning may be successful, motivated by our theorems. Theorem~\ref{TLuD} shows that source information resources with strong dependence on the recipient target can raise the upper bound on performance. Thus, given a source problem and a recipient problem, our heuristic uses the success of an algorithm on the recipient problem after training solely on the source problem and \textbf{not} the recipient problem as a way of assessing potential for successful transfer. Using a classification task, we test whether this heuristic reliably identifies cases where transfer learning works well.

We focused on two similar image classification problems, classifying tigers versus wolves\footnote{\url{http://image-net.org/challenges/LSVRC/2014/browse-synsets}} (TvW) and classifying cats versus dogs\footnote{\url{https://www.kaggle.com/c/dogs-vs-cats-redux-kernels-edition/data}} (CvD). Due to the parallels in these two problem, we expect that a model trained for one task will be able to help us with the other. In our experiment, we used a generic deep convolutional neural network image classification model (VGG16 \cite{simonyan2014very}, using Keras\footnote{\url{https://keras.io/applications/#vgg16}}) to evaluate the aforementioned heuristic to see whether it correlates with any benefit in transfer learning. The table below contains our results: 

\begin{center}
     \begin{tabular}{|c|P{1.5cm}|P{2.5cm}|P{1.5cm}|P{2cm}|P{3cm}|}\hline
    Run & Source Problem & Source Testing \; Accuracy   & Recipient Problem & Additional Training & Recipient Testing \; Accuracy \\
    \hline 
     1 &  CvD  &  84.8\% &  TvW & N &  74.24\%\\
     2 &  CvD  &  84.8\% &  TvW & Y &  95.35\%\\
     3 &  TvW  &  92.16\% &  CvD &  N & 48.36\%\\
     4 &  TvW  &  92.16\% &  CvD &  Y &  82.44\%\\
     \hline 
\end{tabular}
\end{center}

The {\tt Source Problem} column denotes the problem we are transferring from, and the {\tt Recipient Problem} column denotes the problem we are transferring to. The {\tt Source Testing Accuracy} column contains the image classification model's testing accuracy on the source problem after training on its dataset, using a disjoint test dataset. The {\tt Additional Training} column indicates whether we did any additional training before testing the model's accuracy on the recipient problem's dataset --- {\tt N} indicates no training, which means that the following entry in the second {\tt Recipient Testing Accuracy} column contains the results of the heuristic, while {\tt Y} indicates an additional training phase, which means that the following entry in the {\tt Recipient Testing Accuracy} column contains the experimental performance of transfer learning. In each run we start by training our model on the source problem.

Consider Runs 1 and 2. Run 1 is the heuristic run for the CvD $\rightarrow$ TvW transfer learning problem. When we apply the trained CvD model to the TvW problem without retraining, we get a testing accuracy of 74.24\%. This result is promising, as it's significantly above a random fair coin flip, indicating that our CvD model has learned something about the difference between cats and dogs that can be weakly generalized to other images of feline and canine animals. Looking at Run 2, we see that taking our model and training additionally on the TvW dataset yields a transfer learning testing accuracy of 95.35\%, which is higher than the testing accuracy when we train our model solely on TvW (92.16\%). This is an example where transfer learning improves our model's success, suggesting that the pre-training step is helping our algorithm generalize.

When we look at Runs 3 and 4, we see the other side of the picture. The heuristic for the TvW $\rightarrow$ CvD transfer learning problem in Run 3 is a miserable 48.36\%, which is roughly how well we would do randomly flipping a fair coin. It's important to note that this heuristic is not symmetric, which is to be expected --- for example, if the TvW model is learning based on the background of the images and not the animals themselves, we would expect a poor application to the CvD problem regardless of how well the CvD model can apply to the TvD problem. Looking at Run 4, the transfer learning testing accuracy is 82.44\%, which is below the testing accuracy when we train solely on the CvD dataset (84.8\%). This offers some preliminary support for our heuristic --- when the success of the heuristic is closer to random, it may be the case that pre-training not only fails to benefit the algorithm, but can even hurt performance.

Let us consider what insights we can gain from the above results regarding our heuristic. A high value means that the algorithm trained on the source problem is able to perform well on the recipient problem, which indicates that the algorithm is able to identify and discriminate between salient features of the recipient problem. Thus, when we transfer what it learns (e.g., the model weights), we expect to see a boost in performance. Conversely, a low value (around 50\%, since any much lower would allow us to simply flip the labels to obtain a good classifier) indicates that the algorithm is unable to learn features useful for the recipient problem, so we would expect transfer to be unsuccessful. It's important to note that this heuristic is heavily algorithm independent, which is not the case for our theoretical results --- problems with a large degree of latent similarity can receive poor values by our heuristic if the algorithm struggles to learn the underlying features of the problem.

These results offer preliminary support for the suggested heuristic, which was proposed to identify information resources that would be suitable for transfer learning. More research is needed to explore how well it works in practice on a wide variety of problems, which we leave for future work.

\section{Conclusion}

Transfer learning is a type of machine learning that involves a source and recipient problem, where information learned by solving the source problem is used to benefit the process of solving the recipient problem. A popular and potentially lucrative avenue of application is in transferring knowledge from data-rich problems to more niche, difficult problems that suffer from a lack of clean and dependable data. To analyze the bounds of transfer learning, applicable to a large diversity of source/recipient problem pairs, we cast transfer learning into the algorithmic search framework, and define affinity as the degree to which learned information is predisposed towards the recipient problem's target. In our work, we characterize various properties of affinity, show why affinity is essential for the success of transfer learning, and prove results connecting the probability of success of transfer learning to elements of the search framework.

Additionally, we introduce a heuristic to evaluate the likelihood of success of transfer, namely, the success of the source algorithm applied directly to the recipient problem without additional training. Our results show that the heuristic holds promise as a way of identifying potentially transferable information resources, and offers additional interpretability regarding the similarity between the source and recipient problems.

Much work remains to be done to develop theory for transfer learning. Through the results presented here, we learn that there are limits to when transfer learning can be successful, and gain some insight into what powers successful transfer between problems.

\bibliographystyle{plain}
\bibliography{references}

@article{simonyan2014very,
  title={Very deep convolutional networks for large-scale image recognition},
  author={Simonyan, Karen and Zisserman, Andrew},
  journal={arXiv preprint arXiv:1409.1556},
  year={2014}
}

@article{choi2018driving,
  title={Driving experience transfer method for end-to-end control of self-driving cars},
  author={Choi, Dooseop and An, Taeg-Hyun and Ahn, Kyounghwan and Choi, Jeongdan},
  journal={arXiv preprint arXiv:1809.01822},
  year={2018}
}

@inproceedings{wang2015transfer,
  title={Transfer learning for speech and language processing},
  author={Wang, Dong and Zheng, Thomas Fang},
  booktitle={2015 Asia-Pacific Signal and Information Processing Association Annual Summit and Conference (APSIPA)},
  pages={1225--1237},
  year={2015},
  organization={IEEE}
}

@book{cover2012elements,
  title={Elements of information theory},
  author={Cover, Thomas M and Thomas, Joy A},
  year={2012},
  publisher={John Wiley \& Sons}
}

@inproceedings{lauw2020bias,
  author    = {Julius Lauw and
               Dominique Macias and
               Akshay Trikha and
               Julia Vendemiatti and
               George D. Monta{\~{n}}ez},
  title     = {The {B}ias-{E}xpressivity {T}rade-off},
  booktitle = {Proceedings of the 12th International Conference on Agents and Artificial
               Intelligence, Volume 2},
  pages     = {141--150},
  publisher ={{SCITEPRESS}},
  year ={2020},
  editor    = {Ana Paula Rocha and
               Luc Steels and
               H. Jaap van den Herik},
  url       = {https://doi.org/10.5220/0008959201410150},
  doi       = {10.5220/0008959201410150},
  bibsource = {dblp computer science bibliography, https://dblp.org}
}

@article{dembski2009conservation,
  title={Conservation of information in search: measuring the cost of success},
  author={Dembski, William A and Marks II, Robert J},
  journal={IEEE Transactions on Systems, Man, and Cybernetics-Part A: Systems and Humans},
  volume={39},
  number={5},
  pages={1051--1061},
  year={2009},
  publisher={IEEE}
}

@inproceedings{sam2020decomposable,
  author    = {Tyler Sam and
               Jake Williams and
               Abel Tadesse and
               Huey Sun and
               George D. Monta{\~{n}}ez},  title     = {Decomposable Probability-of-Success Metrics in Algorithmic Search},
  booktitle = {Proceedings of the 12th International Conference on Agents and Artificial
               Intelligence, Volume 2},
  pages     = {785--792},
  publisher = {{SCITEPRESS}},
  editor    = {Ana Paula Rocha et al.},
  year      = {2020},
  url       = {https://doi.org/10.5220/0009098807850792},
  doi       = {10.5220/0009098807850792},
  bibsource = {dblp computer science bibliography, https://dblp.org}
}

@inproceedings{montanez2019fobfl,
  title={The {F}utility of {B}ias-{F}ree {L}earning and {S}earch},
  author={Monta{\~n}ez, George D and Hayase, Jonathan and Lauw, Julius and Macias, Dominique and Trikha, Akshay and Vendemiatti, Julia},
  booktitle={32nd Australasian Joint Conference on Artificial Intelligence},
  pages={277--288},
  year={2019},
  organization={Springer}
}

@inproceedings{Montanez2016TheFO,
  title="The {F}amine of {F}orte: {F}ew {S}earch {P}roblems {G}reatly {F}avor {Y}our {A}lgorithm",
  author="{Monta{\~n}ez}, George D",
  booktitle="Systems, Man, and Cybernetics (SMC), 2017 IEEE International Conference on",
  pages="477--482",
  year="2017a",
  organization="IEEE"
}

@phdthesis{montanez2017machine,
  title="Why {M}achine {L}earning {W}orks",
  author="George D. Monta{\~{n}}ez",
  year="2017b",
  school="Carnegie Mellon University",
  label="whyML"
}

@article{Wolpert1997NoFL,
  title={No free lunch theorems for optimization},
  author={David H. Wolpert and William G. Macready},
  journal={IEEE Trans. Evolutionary Computation},
  year={1997},
  volume={1},
  pages={67-82}
}

@article{SchafferConservation,
  title={A {C}onservation {L}aw for {G}eneralization {P}erformance},
  author={Cullen Schaffer},
  journal={Machine Learning Proceedings 1994},
  year={1994},
  volume={1},
  pages={259-265}
}

@article{pan2009survey,
  title={A survey on transfer learning},
  author={Pan, Sinno Jialin and Yang, Qiang},
  journal={IEEE Transactions on knowledge and data engineering},
  volume={22},
  number={10},
  pages={1345--1359},
  year={2009},
  publisher={IEEE}
}

\pagebreak
\section*{Appendix: Proofs}

\consofaffin*

\begin{proof}
        Note that $\sum_{\mathbf{t}}\mathbf{t}$ is the sum of all target vectors definable on $\mathrm{\Omega}$, which themselves correspond to the nonempty subsets of $\mathrm{\Omega}$. Thus, the sum equals a constant vector, $c \cdot \bm{1} = [c, c, \ldots, c]^{\top}$ where $c = 2^{|\mathrm{\Omega}|-1}$.
        
        By the definition of affinity and the linearity of expectation, we have
        \begin{align*}
            \sum_{\mathbf{t}} \af(\D, \mathbf{t}, \fro) 
            &= \sum_{\mathbf{t}}[ \E_\D[\mathbf{t}^{\top}\mathbf{P}_{\phi, \froL}]- \mathbf{t}^{\top}\mathbf{P}_{\phi, \fro}] \\
            &= \sum_{\mathbf{t}} \E_\D[\mathbf{t}^{\top}\mathbf{P}_{\phi, \froL}]- \sum_{\mathbf{t}}\mathbf{t}^{\top}\mathbf{P}_{\phi, \fro} \\
            &=  \left(\sum_{\mathbf{t}}\mathbf{t}^{\top}\right)\E_\D[\mathbf{P}_{\phi, \froL}]- \left(\sum_{\mathbf{t}}\mathbf{t}^{\top}\right)\mathbf{P}_{\phi, \fro} \\
            &=  (c \cdot \bm{1}^{\top})\E_\D[\mathbf{P}_{\phi, \froL}]- (c \cdot \bm{1}^{\top})\mathbf{P}_{\phi, \fro} \\
            &=  c \cdot (\bm{1}^{\top}\E_\D[\mathbf{P}_{\phi, \froL}])- c \cdot (\bm{1}^{\top}\mathbf{P}_{\phi, \fro}) \\
            &= c - c = 0
        \end{align*}
        where the third equality follows from the fact that neither $\E_\D[\mathbf{P}_{\phi, \froL}]$ nor $\mathbf{P}_{\phi, \fro}$ is a function of $\mathbf{t}$, allowing both to be pulled out of their sums, and the  penultimate equality follows from the linearity of expectation and the fact that $\bm{1}^{\top}\mathbf{P} = 1$ for any probability mass vector $\mathbf{P}$.
\end{proof}

\begin{restatable}{lemma}{mutualInfoLem}\label{lem:MI}
If $I( \FRoL ; T_R) \leq I( \FRo , L; T_R)$  then
\[
I(\FRoL; T_R) \leq I(F_S; T_R) + I(\FRo; T_R).
\]
\end{restatable}

\begin{proof}
\begin{align*}
 I(\FRoL; T_R) &\leq I(\FRo, L; T_R) \\
&= I(L; T_R \mid \FRo) + I(\FRo; T_R) \\
&= H(L \mid \FRo) - H(L \mid \FRo, T_R) + I(\FRo; T_R)  \\
&= H(L \mid \FRo) - H(L \mid T_R) + I(\FRo; T_R)  \\
&\leq H(L) - H(L \mid T_R) + I(\FRo;T_R)\\
&= I(L;T_R) + I(\FRo;T_R) \\
&\leq I(F_S; T_R) + I(\FRo; T_R)
\end{align*} 
where the first equality follows from application of the chain rule for mutual information, the second and fourth equalities follow from the definition of mutual information, the third equality follows from the conditional independence assumption, and the final inequality follows by application of the Data Processing Inequality~\cite{cover2012elements}.
\end{proof}

\tld*
\begin{proof}
         By d-separation of the graphical model structure in Figure~\ref{DSTL} and the Data Processing Inequality~\cite{cover2012elements}, we have that $I(\FRoL; T_R) \leq I(\FRo, L; T_R)$.
         Applying the result from Lemma~\ref{lem:MI} to the Learning Under Dependence theorem \cite{sam2020decomposable}, we obtain
         \begin{align*}
              \phi_{TL} &\leq \frac{I(F_{R+L}; T_R) + D(P_{T_R} \| \mathcal{U}_{T_R}) + 1}{I_{\mathrm{\Omega}}}\\
              &\leq \frac{I(F_S; T_R) + I(F_R; T_R) + D(P_{T_R} \| \mathcal{U}_{T_R}) + 1}{I_{\mathrm{\Omega}}}.
         \end{align*}
 \end{proof}
 
 \fofli*
 \begin{proof}
We seek to bound the proportion of successful search problems for which $\g(t, f) \geq \phi_{\mathrm{min}}$ for any threshold $\phi_{\mathrm{min}} \in (0, 1]$. \ Then, 
        \begin{align*}
            \frac{|\mathcal{B}_{\phi_{\mathrm{min}}}|}{|\mathcal{B}|} &= \frac{1}{ |\mathcal{B}|} \sum_{l \in \mathcal{B}} \mathds{1}_{\g(t,\frol) \geq \phi_{\mathrm{min}}}\\
                                               &=  \mathbb{E}_{\mathcal{U}[\mathcal{B}]}[\mathds{1}_{\g(t,\froL) \geq \phi_{\mathrm{min}}}] \\
                                               &= \Pr(\g(t, \froL) \geq \phi_{\mathrm{min}})\\
                                            &=\Pr(\mathbf{t}^{\top} \mathbf{P}_{\g, \froL} \geq \phi_\mathrm{min})
        \end{align*}
        where the final equality follows from the definition of \dms.
        
        Note that all of the randomness in $\froL$ comes from the learned information, $L$, and not the fixed \r{} information resource $R_o$. Applying Markov's Inequality and the definition of $\af(\mathcal{D}_L,\A, \mathbf{t})$, we obtain
        \begin{align*}
            \frac{|\mathcal{B}_{\phi_{\mathrm{min}}}|}{|\mathcal{B}|} &\leq \frac{\mathbb{E}_{\mathcal{U}[\mathcal{B}]} [\mathbf{t}^{\top} \mathbf{P}_{\g,  \froL}]}{\phi_{\mathrm{min}}} \\
                                               &= \frac{\g(t, \fro) + \af(\mathcal{U}[\B],\mathbf{t}, \fro)}{\phi_{\mathrm{min}}}.
        \end{align*}
    \end{proof}
    
\futility*
\begin{proof}
        Let $\mathcal{L}$ be the space of possible learning resources. Then,
        \begin{align*}
            \Pr(\omega \in t; \mathcal{A}_L) 
                &= \int_\mathcal{L} \Pr(\omega \in t, l; \mathcal{A}) \dif l\\
                &= \int_\mathcal{L} \Pr(\omega \in t \mid l; \mathcal{A})\Pr(l) \dif l.
        \end{align*}
        Since we are considering the general $\phi$ probability of success for algorithm $\mathcal{A}$ on $t$ using learning resource $l$, but with a fixed \r{} information resource $\fro$, we have
        \[
          \Pr(\omega \in t \mid l; \mathcal{A}) = P_{\g, \fro}(\omega \in t \mid l) = P_{\g, \frol}(\omega \in t).
        \]
        Also note that $\Pr(l) = \mathcal{D}_L(l)$ because our information resources are drawn from the distribution $\mathcal{D}_L$. Making these substitutions, we obtain
        \begin{align*}
            \Pr(\omega \in t; \mathcal{A}_L) 
                &= \int_\mathcal{L} P_{\g, \frol}(\omega \in t)\mathcal{D}_L(l) \dif l\\
                &= \mathbb{E}_{\mathcal{D}_L}\left[P_{\g, \froL}(\omega \in t)\right]\\
                &= \mathbb{E}_{\mathcal{D}_L}\left[\mathbf{t}^{\top}\textbf{P}_{\g,\froL}\right]\\
                &= \af(\mathcal{D}_L, \bm{t}, \fro) + \mathbf{t}^{\top}\textbf{P}_{\g,\fro}\\
                &= \phi(t,\fro).
        \end{align*}
    \end{proof}

\begin{restatable}[Equivalence of Affinity]{lemma}{equivAffin}\label{lem:equiaf}
Given a fixed \r{} problem ($\Omega, t, \fro$), where $t$ has corresponding target function $\mathbf{t}$, a finite set of learning resources $\B$, and a set $\mathcal{P} = \{\mathcal{D}\mid \mathcal{D}\in \mathbb{R}^{|\B|}, \sum_{l \in \B} \mathcal{D}(l) = 1\}$ of all discrete $|\B|$-dimensional simplex vectors,
\begin{align*}
    \E_{\unif}[\af(\mathcal{D},\mathbf{t},\fro)] = \af(\mathcal{U}[\B],\mathbf{t},\fro)
\end{align*}
where $\mathcal{D} \sim \unif$.
\end{restatable}

\begin{proof}
Let $L \sim \D$. Then, 
\begin{align*}
    & \E_{\unif}[\af(\D,\textbf{t},\fro)]\\
    &\phantom{+++}= \E_{\unif}[\E_{\D}[\mathbf{t}^{\top}\textbf{P}_{\g,\froL}] - \phi(t, \fro)] \\
    &\phantom{+++}= \E_{\unif}\left[\sum_{l \in \B} \D(l)\mathbf{t}^{\top}\textbf{P}_{\g,\frol}\right] - \phi(t, \fro)\\
    & \phantom{+++}= \sum_{l \in \B}\mathbf{t}^{\top}\textbf{P}_{\g,\frol}\E_{\unif}[\D(l)] - \phi(t, \fro)
\end{align*}
The quantity $\E_{\unif}[\D(l)]$ is a uniform expectation on the amount of mass that the random distribution $\D$ places on resource $l$. Since $\mathcal{P}$ contains all possible distributions over $\B$, under uniform expectation the same amount of probability mass gets placed on each information resource. So, $\E_{\unif}[\D(i)] = \E_{\unif}[\D(j)]$ for any $i, j \in \B$. Since the probability mass on any two learning resources is equivalent and the total probability mass must sum to one, by the Expectation of Simplex Vectors is Simplex~\cite{montanez2019fobfl}, we have $\E_{\unif}[\D(f)] = \frac{1}{|\B|}$. Continuing,
\begin{align*}
    \E_{\unif}[\af(\D,\textbf{t},\fro)] &= \frac{1}{|\B|}\sum_{l \in \B} \mathbf{t}^{\top}\textbf{P}_{\g,\froL}  - \phi(t, \fro) \\
    &= \af(\mathcal{U}[\B], \textbf{t}, \fro).
\end{align*}
\end{proof}

\fofad*
\begin{proof}
        Let $\D \sim \unif$. Then,
        \begin{align*}
            \frac{\mu (\mathcal{G}_{\mathbf{t}, \phi_{\mathrm{min}}})}{\mu (\mathcal{P})} &= \Pr(\af(\D, \mathbf{t}, \fro) \geq \phi_{\mathrm{min}}) \\
            &= \Pr[\phi(t,\fro) + \af(\D, \mathbf{t}, \fro) \geq \phi(t,\fro) + \phi_{\mathrm{min}}] \\
            &= \Pr[\E_{\D}[\textbf{t}^{\top}{\bf P}_{\g,\froL}] \geq \phi(t,\fro) + \phi_{\mathrm{min}}].
        \end{align*}
        
        Applying Markov's Inequality and Lemma~\ref{lem:equiaf}, we obtain
        \begin{align*}
            \frac{\mu (\mathcal{G}_{\mathbf{t}, \phi_{\mathrm{min}}})}{\mu (\mathcal{P})} &\leq \frac{\E_{\unif}[\E_{\D}[\textbf{t}^{\top}{\bf P}_{\g,\froL}]]}{\phi(t,\fro) + \phi_{\mathrm{min}}} \\
            &= \frac{\phi(t,\fro) + \E_{\unif}[\af(\mathcal{D}, \mathbf{t}, \fro)]}{\phi(t,\fro) + \phi_{\mathrm{min}}} \\ 
            &= \frac{\phi(t,\fro) + \af(\mathcal{U}[\B], \mathbf{t}, \fro)}{\phi(t,\fro) + \phi_{\mathrm{min}}} \\ 
            &\leq \frac{\phi(t,\fro) + \af(\mathcal{U}[\B], \mathbf{t}, \fro)}{\phi_{\mathrm{min}}}.
        \end{align*}
        
\end{proof}

\sa*

\begin{proof}
        \begin{align*}
            |\g_{TL} - \g_{NoTL}| &= |\textbf{t}^{\top}(\mathbf{P}_{TL}-\mathbf{P}_{NoTL})|\\
            &= |\sum_{\omega}\mathds{1}_{\omega \in T}(\mathbf{P}_{TL}(\omega)-\mathbf{P}_{NoTL}(\omega))|\\
            & \leq |T| \sup_{w \in T} |\mathbf{P}_{TL}(\omega)-\mathbf{P}_{NoTL}(\omega)|\\
            &\leq |T|\sqrt{\frac{1}{2}D_{KL}(\mathbf{P}_{TL}||\mathbf{P}_{NoTL})}
        \end{align*}
        where the first equality follows form the definition of decomposable probability of success metrics and the final inequality follows by application of Pinsker's Inequality.
\end{proof}

\end{document}